\documentclass[sigconf, nonacm, pdfa, draft=false]{acmart}

\newcommand\vldbdoi{10.14778/3717755.3717766}

\newcommand\vldbavailabilityurl{https://github.com/Gnaiqing/WeShap}
\newcommand\vldbpagestyle{plain} 

\usepackage[english]{babel}
\usepackage{amsthm}
\usepackage{dsfont}
\usepackage{multirow}
\usepackage{graphicx}
\usepackage{caption}
\usepackage{subcaption}
\usepackage[linesnumbered,ruled,vlined]{algorithm2e}
\usepackage{balance}
\SetKwInput{KwInput}{Input}                
\SetKwInput{KwOutput}{Output}              
\usepackage[normalem]{ulem}
\useunder{\uline}{\ul}{}

\newtheorem{theorem}{Theorem}
\newtheorem{definition}{Definition}
\newtheorem*{example}{Example}

\usepackage[utf8]{inputenc}
\usepackage[english]{babel}
\usepackage{colorprofiles}
\usepackage[a-2b,mathxmp]{pdfx}[2018/12/22]
\hypersetup{pdfstartview=}

\begin{document}
\title{WeShap: Weak Supervision Source Evaluation with Shapley Values}

\author{Naiqing Guan}
\affiliation{%
  \institution{University of Toronto}
  \city{Toronto}
  \country{Canada}
}
\email{naiqing.guan@mail.utoronto.ca}

\author{Nick Koudas}
\affiliation{%
  \institution{University of Toronto}
  \city{Toronto}
  \country{Canada}
}
\email{koudas@cs.toronto.edu}

\begin{abstract}
Efficient data annotation stands as a significant bottleneck in training contemporary machine learning models. The Programmatic Weak Supervision (PWS) pipeline presents a solution by utilizing multiple weak supervision sources to automatically label data, thereby expediting the annotation process. Given the varied contributions of these weak supervision sources to the accuracy of PWS, it is imperative to employ a robust and efficient metric for their evaluation. This is crucial not only for understanding the behavior and performance of the PWS pipeline but also for facilitating corrective measures.

In this paper, we introduce WeShap values as an evaluation metric. This metric quantifies the average contribution of weak supervision sources within a proxy PWS pipeline, leveraging the theoretical underpinnings of Shapley values. We demonstrate efficient computation of WeShap values using dynamic programming, achieving quadratic computational complexity relative to the number of weak supervision sources. 

Our experiments demonstrate the versatility of WeShap values across various applications, including the identification of beneficial or detrimental labeling functions, refinement of the PWS pipeline, comprehension of the pipeline's behavior, and scrutinizing specific instances of mislabeled data. Although initially derived from a specific proxy PWS pipeline, we empirically demonstrate the generalizability of WeShap values to other PWS pipeline configurations. Our findings indicate a noteworthy average improvement of 5.0 points in downstream model accuracy through the revision of the PWS pipeline compared to previous state-of-the-art methods, underscoring the efficacy of WeShap values in enhancing data quality for training machine learning models.
\end{abstract}

\maketitle

\pagestyle{\vldbpagestyle}

\ifdefempty{\vldbavailabilityurl}{}{
\vspace{.3cm}
\begingroup\small\noindent\raggedright\textbf{PVLDB Artifact Availability:}\\
The source code, data, and/or other artifacts have been made available at \url{\vldbavailabilityurl}.
\endgroup
}

\section{Introduction}
Deploying modern machine learning models in new application scenarios often hinges on the availability of large annotated datasets, posing a significant bottleneck. While obtaining unlabeled data is relatively straightforward, annotating typically demands substantial effort and financial resources. Programmatic weak supervision (PWS) \cite{ratner2017snorkel,zhang2022survey} offers a promising avenue for mitigating this manual annotation burden. In the PWS paradigm, instead of annotating individual instances, users concentrate on developing multiple weak supervision sources capable of automatically annotating a portion of the data in a noisy manner. These sources may originate from crowdsourced data  \cite{lan2019learning}, human-designed heuristics \cite{ratner2017snorkel,yu2020fine}, or pre-trained models \cite{bach2019snorkel,smith2022language}. Represented as labeling functions (LFs), these weak supervision sources assign weak labels to some data points while abstaining from others. Given the potential contradictions among weak labels, a label model is trained to denoise and aggregate them. Subsequently, the labels predicted by the label model are utilized to train the downstream machine learning model, also called the end model.

Within the PWS framework, labeling functions (LFs) are the fundamental components for automated annotation, profoundly influencing data quality and downstream machine learning model accuracy. Despite extensive research on the efficient design of LFs, with notable studies 
\cite{varma2018snuba,hsieh2022nemo,boecking2020interactive,denham2022witan},
 scant attention has been devoted to LF quality evaluation methodologies. Conventionally, LF evaluation relies on metrics such as accuracy (evaluated on a withheld labeled dataset), coverage, or confliction with other LFs. For instance, the Snorkel library \cite{ratner2017snorkel}'s \textit{LFAnalysis} module reports these metrics. However, these metrics gauge different facets of LFs, often yielding conflicting assessments. Notably, while both high coverage and accuracy are desirable, they typically exhibit a negative correlation in practice. 
 Boeckling et al. \cite{boecking2020interactive} propose ranking LFs based on $(2\alpha_j-1)*l_j$, where $\alpha_j$ and $l_j$ denote the accuracy and coverage of LF $j$, respectively. They additionally present a theorem asserting that under assumptions of LF independence conditional on class labels and uniformly distributed LF errors, this ranking minimizes the expected risk of a certain label model on training data. However, these stringent assumptions are rarely met in practice. Furthermore, along with all aforementioned metrics, their metric overlooks the data distribution in feature space, which can significantly impact downstream model performance.

\begin{figure}[tbp]
\centering
\includegraphics[width=\columnwidth]{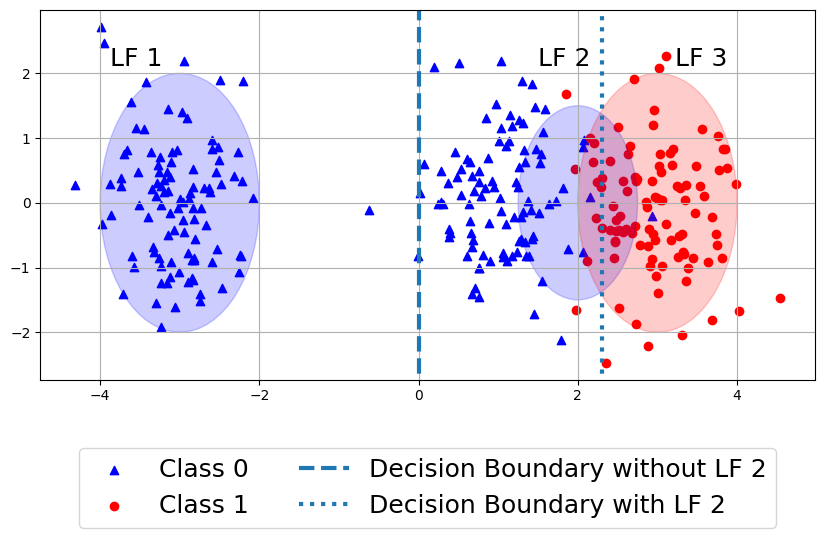}
 \caption{A motivating example for comprehensive LF evaluation. LF 2 is essential for reducing classification errors, although its accuracy is around 0.5. }\label{fig:motivate-example}
\end{figure}

Figure \ref{fig:motivate-example} provides a motivating example to illustrate the limitation of existing LF evaluation metrics. There are three LFs for a binary classification task, where LF 1 and LF 3 have over 90\% accuracy and 30\% coverage, while LF 2 has only around 50\% accuracy and 20\% coverage. However, omitting LF 2 during the label model training shifts the downstream model's decision boundary to $x=0$, leading to substantial classification errors. The key is that although LF 2 has lower accuracy and coverage compared to other labeling functions, it lies near the optimal decision boundary, making a greater contribution to the downstream model accuracy.

The provided example underscores the necessity for more nuanced LF evaluation metrics to assess each LF's contribution on downstream model accuracy directly. Zhang et al.'s work on Source-aware Influence Functions (SIF) \cite{zhang2022understanding} pioneered the exploration of comprehensive LF evaluation metrics by learning to evaluate LFs based on their influence. While SIF considers the data distribution in feature space, it falls short of directly measuring an LF's impact on the accuracy of the downstream model.

A well-established approach for allocating contributions in cooperative games is the Shapley Value \cite{shapley1953value}, renowned for its fairness, symmetry, and efficiency properties \cite{chalkiadakis2022computational}. Conceptually, if we regard each LF as a player in a cooperative game and utilize downstream model accuracy to gauge utility, then the Shapley value for each LF mirrors its contribution to the downstream model. However, direct computation of the Shapley value entails $\mathcal{O}(2^m)$ computations, where $m$ signifies the number of players in the game (equivalent to the number of LFs in our context). Moreover, evaluating a coalition of LFs proves computationally intensive, necessitating training the label model and downstream model from scratch. Consequently, direct computation of the Shapley value becomes impractical. This poses a fundamental challenge for LF evaluation based on Shapley value: \textit{how can we efficiently compute Shapley values to evaluate LFs within a PWS framework?}

In this paper, we address the above challenge by proposing the WeShap value (\textbf{We}akly Supervised \textbf{Shap}ley Value), which measures the \textit{exact shapley value} of LFs in a specific PWS framework, with majority voting being the label model and K-nearest neighbors (KNN) being the downstream model. We demonstrate that the Shapley value of LFs can be computed efficiently under these specific model choices. The WeShap value can be used to identify helpful or harmful LFs, revise PWS pipelines, and understand model behaviors. Our experiments show that while the WeShap value is computed based on specific model choices, it also attains excellent performance when other models, such as Snorkel \cite{ratner2017snorkel}, BERT \cite{devlin2018bert}, or ResNet \cite{he2016deep} are utilized in the pipeline.

\section{Preliminaries}
In this section, we provide background material on setting up the programmatic weak supervision framework and introduce Shapley values and their favorable properties.

\subsection{Programmatic Weak Supervision}
We consider C-way classification scenarios, where $\mathcal{X}$ represents the feature space, $\mathcal{Y}$ denotes the label space, and $\mathcal{P}$ stands for the underlying data distribution. Within the PWS framework, the user possesses a training dataset $D_{Train}=\{x_i\}_{i=1}^n$, where the corresponding labels $\{y_i\}_{i=1}^n$ are initially unknown. Optionally, a smaller validation set $D_{Valid}=\{x_i, y_i\}_{i=n+1}^{n+n_v}$ may exist for hyperparameter tuning. Both training and validation datasets adhere to the underlying distribution, i.e., $(x_i, y_i) \sim \mathcal{P}$. The objective is to train a machine learning model $f$ to minimize its expected risk on $\mathcal{P}$.

To attain this objective, the user crafts a set of labeling functions (LFs) $\Lambda=\{\lambda_j\}_{j=1}^m$, where each LF furnishes noisy labels (termed weak labels) to a subset of data. We denote by $L_{ij}=\lambda_j(x_i)$ the weak label provided by $\lambda_j$ for $x_i$. Here, $L_{ij}\in \{\mathcal{Y}\cup \{\emptyset\}\}$, where $\emptyset$ indicates that $\lambda_j$ abstains from labeling $x_i$. Subsequently, the weak labels are employed to train a label model $\Theta$, which gauges the accuracies of each LF and predicts a single probabilistic label $\tilde{y}_i$ per instance. Finally, the training dataset, coupled with the generated labels, is utilized to train the downstream model $f$.

\subsection{Shapley Value} \label{sec:shapley}
The Shapley value is based on cooperative game theory \cite{elkind2016cooperative}. Formally, a cooperative game is defined by a pair $(I,v)$, where $I=\{1,...,m\}$ denotes the set of players and $v:2^I\to \mathbb{R}$ is the utility function, which assigns a real value $v(S)$ to every coalition $S\subset I$. Furthermore, the utility of an empty coalition is set to 0, i.e., $v(\emptyset)=0$. Intuitively, the utility function defines how much payoff a set of players can achieve by forming a coalition. One central question in cooperative game theory is how to distribute the total payoff fairly among the players. The Shapley value \cite{shapley1953value} is a classical solution to this question, which assigns each player their average marginal contribution to the value of the predecessor set over every permutation of the player set. Formally, the Shapley value is defined as
\begin{equation} \label{eq:shap1}
    \phi_j^{Sh} = \frac{1}{|\Pi(I)|}\sum_{\pi\in \Pi(I)}\left[v(P_j^\pi\cup \{j\})-v(P_j^\pi)\right]
\end{equation}

Where $\Pi(I)$ denotes all possible permutations of $I$ and $P_j^\pi$ denotes the predecessor set of player $j$ in permutation $\pi$. Intuitively, suppose the players join the coalition randomly; the Shapley value for player j would be the expectation of their marginal contribution to the payoff.  An equivalent formulation is
\begin{equation} \label{eq:shap2}
    \phi_j^{Sh} = \frac{1}{m}\sum_{S\in I/\{j\}} \frac{1}{\binom{m-1}{|S|}}\left[v(S\cup \{j\})-v(S)\right]
\end{equation}

The Shapley value is theoretically appealing as it is the unique solution that satisfies the following desiderata simultaneously:

\begin{description}
    \item[Efficiency:]  The payoff of the full player set is completely distributed among all players, i.e., $\sum_{j\in I}\phi_j = v(I)$.

    \item[Null Player:] If a player contributes nothing to each coalition, then they should receive zero value, i.e., $[\forall S|j\notin S, v(S\cup \{j\})=v(S)]\Rightarrow \phi_j=0$.

    \item[Symmetry:] If two players play equal roles to each coalition, they should receive equal value, i.e., $[\forall S|i,j \notin S, v(S\cup \{i\})=v(S\cup \{j\}]\Rightarrow \phi_i=\phi_j$.

    \item[Additivity:] Given two coalition games $(I,v)$ and $(I,w)$ with different utility functions, the value a player receives under a coalition game $(I, v+w)$ is the sum of the values they receive under separate coalition games, i.e., $\phi_j(v+w)=\phi_j(v)+\phi_j(w)$.

\end{description}

In our context, the set of LFs corresponds to the players, while the accuracy of the downstream model on a holdout dataset, utilizing a coalition of LFs to label the training set, serves as the utility function. Under this framework, several desirable properties of the Shapley value emerge:
\begin{description}
 \item[Efficiency:] The accuracy of the downstream model is entirely attributed to all LFs.
 \item[Null player:] LFs contributing nothing (e.g., abstaining on all data) receive a score of zero.
 \item[Symmetry:] LFs contributing equally to downstream model accuracy receive identical scores.
 \item[Additivity:] This property facilitates efficient score computation when the downstream model is employed across multiple applications or datasets.
\end{description}
These properties are crucial for fair LF evaluation within the PWS framework, rendering the Shapley value an appealing solution.

\section{Our Method} \label{sec:method}
In this section, we first establish the formalization of the proxy PWS framework, then outline the cooperative game within this proxy framework. Subsequently, we derive the WeShap value, representing the Shapley value of LFs within the specified cooperative game structure. Next, we present an efficient approach for computing the WeShap value using dynamic programming and analyze its computational complexity. Finally, we introduce various application scenarios showcasing the utility of the WeShap value.

\begin{figure*}[tbp]
\centering
 \includegraphics[width=2\columnwidth]{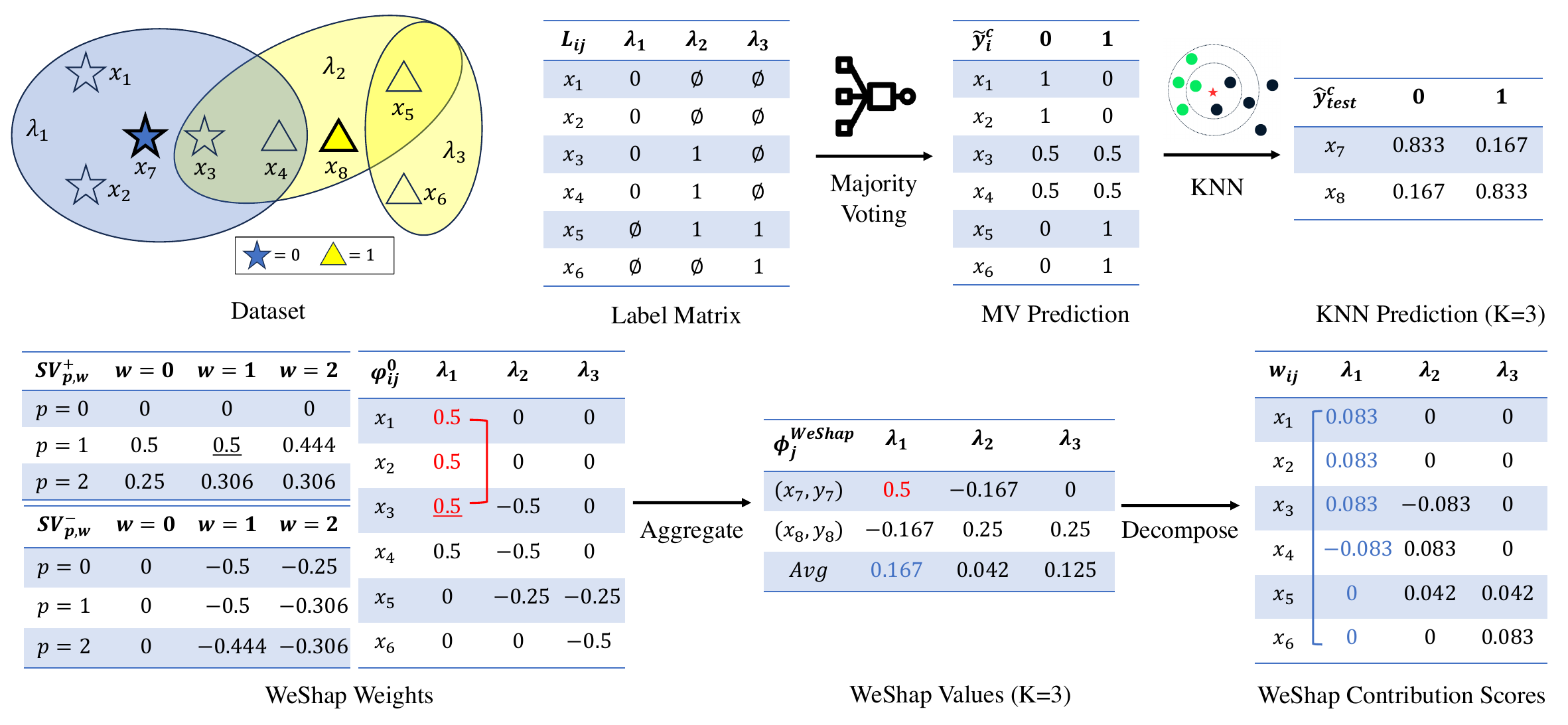}
 \caption{Illustration of the WeShap value computation.}\label{fig:weshap}
\end{figure*}

\subsection{Proxy Framework} \label{sec:proxy}
\begin{definition}[Proxy Framework]\label{def:proxy-framework}
The proxy framework is a programmatic weak supervision framework, where the label model uses majority voting (MV) to aggregate LF outputs, and the downstream model uses the K-nearest neighbors (KNN) algorithm to make predictions.
\end{definition}

For an LF set $\Lambda$, let $\Theta_{\Lambda}^c(x)$ denote the MV model's predicted probability that unlabeled training data $x$ belongs to class $c$:

\begin{equation}\label{eq:mv-pred}
\Theta_{\Lambda}^c(x) = \left\{
\begin{array}{lr}
\frac{\sum_{\lambda\in \Lambda}\mathds{1}(\lambda(x)=c)}{\sum_{\lambda\in \Lambda}\mathds{1}(\lambda(x)\neq \emptyset)}, & \sum_{\lambda\in \Lambda}\mathds{1}(\lambda(x)\neq \emptyset) > 0 \\
\frac{1}{C}, & \sum_{\lambda\in \Lambda}\mathds{1}(\lambda(x)\neq \emptyset) = 0
\end{array}
\right.
\end{equation}
In other words, if LFs are activated on the instance, the label model uses majority voting to predict its probabilistic label; otherwise, it predicts every class with equal probabilities. 

The downstream model in the proxy framework is a KNN classifier, where the user can arbitrarily specify the choice of K, the distance metric, and the weight function. Here, we assume that uniform weights are applied for simplicity. For a validation instance $x_{val}$, let $\mathcal{N}_K(x_{val})$ contains the K nearest neighbors of $x_{val}$ in the training set. The predicted probability of $x_{val}$ belongs to class $c$ is:
\begin{equation} \label{eq:knn-pred}
f_{\Lambda}^c(x_{val})=\frac{\sum_{i\in \mathcal{N}_K(x_{val})}\Theta_{\Lambda}^c(x_i)}{K}
\end{equation}

Our framework addresses multiclass scenarios but does not encompass multi-label classification, where multiple non-exclusive labels can be assigned to each instance. Consequently, the probabilities assigned to each class in Equation \ref{eq:knn-pred} sum to 1.

The accuracy of the downstream model is measured on a holdout dataset $D_{val}\sim \mathcal{P}$. Note that the accuracy defined in Equation \ref{eq:risk} differs from traditional classification accuracy, as it is based on probabilistic predictions rather than categorical ones.
\begin{equation}\label{eq:risk}
    Acc_{f, \Lambda}(D_{val}) = \frac{\sum_{(x,y)\in D_{val}} f_{\Lambda}^y(x)}{|D_{val}|}
\end{equation}

\begin{definition}[Proxy Game]\label{def:proxy-game}
The proxy game is a cooperative game $(I,v)$ defined under the proxy framework, where  $I=\{1,...,m\}$ denotes the set of LFs, and $v:2^I\rightarrow \mathds{R}$ maps an LF coalition to the accuracy gain of the proxy framework (measured by Equation \ref{eq:risk}) using that LF coalition compared to random prediction (i.e. with accuracy $\frac{1}{C}$).
\end{definition}

We use the accuracy gain instead of the absolute accuracy of the proxy framework to define the utility function $v$ to guarantee $v(\emptyset) = 0$, which is a prerequisite for Shapley value computation. 

\subsection{WeShap Value} \label{sec:weshap}
Within the proxy framework, we proceed to derive the formulation of the WeShap value, representing the Shapley values of LFs within the defined cooperative game structure. 

First, we consider the marginal utility of an LF $\lambda$ in classifying a data point $(x,y)$ with the MV label model. Notice that LFs that abstain on $(x,y)$ will not affect the prediction of the MV model, so we only consider LFs providing weak labels. Without loss of generality, we first consider the case where $\lambda(x)=y$. Suppose in a LF permutation, there are $a$ correct LFs and $b$ wrong LFs in the predecessor LF set of $\lambda$. The marginal utility of $\lambda$ on $(x,y)$ is:

\begin{equation}
\psi(a,b) = \left\{
    \begin{array}{lr}
    \frac{a+1}{a+b+1} - \frac{a}{a+b}, & a+b>0 \\
    1 - \frac{1}{C}, & a+b=0
\end{array}
\right.
\end{equation}

As a result, we can efficiently compute the average marginal contribution of $\lambda$ by enumerating the number of accurate and inaccurate LFs in the predecessor set of $\lambda$ over all possible LF permutations.

\begin{theorem} \label{theorem:mv}
Consider a coalition game $\mathcal{G}$ where a majority voting (MV) model utilizes a set of LFs to label a data point $(x,y)$. The player set $I=\{1,...,m\}$ denotes the set of LFs, and the utility function $v'$ maps an LF coalition to the accuracy gain of the MV model using the LFs compared to random prediction with accuracy $\frac{1}{C}$. Suppose  there are $p$ correct LFs and $w$ incorrect LFs on $(x,y)$, then the Shapley value of a correct LF $\lambda$ in the coalition game $\mathcal{G}$ is: 

\begin{equation}\label{eq:sv-pos}
{SV}_{p,w}^+=\frac{1}{p+w}\sum_{i=0}^{p-1}\sum_{j=0}^{w}\left[\psi(i,j)\frac{\binom{p-1}{i}\binom{w}{j}}{\binom{p+w-1}{i+j}}\right]
\end{equation}
\end{theorem}

\begin{proof}
Consider all permutations of LFs that produce weak labels on $(x,y)$. In a permutation with $i$ correct LFs and $j$ incorrect LFs before $\lambda$, the marginal utility of $\lambda$ is $\psi(i,j)$. The number of permutations with $i$ correct LFs and $j$ incorrect LFs before $\lambda$ is $\binom{p-1}{i}\binom{w}{j}(i+j)!(p+w-1-i-j)!$, as we need to select $i$ correct LFs from the other $p-1$ correct LFs, $j$ incorrect LFs from the $w$ incorrect LFs, and consider the order of LFs in the permutation. Following the definition of Shapley value in Equation \ref{eq:shap1}, the Shapley value of $\lambda$ in $\mathcal{G}$ is
\begin{equation}\label{eq:permutation_shap}
\sum_{i=0}^{p-1}\sum_{j=0}^w\left[\frac{\psi(i,j)*\binom{p-1}{i}\binom{w}{j}(i+j)!(p+w-1-i-j)!}{(p+w)!}\right]
\end{equation}

This equals to ${SV}_{p,w}^+$ expressed in Equation \ref{eq:sv-pos}.
\end{proof}

As the Shapley value has the efficiency property, we can calculate the Shapley value of an inaccurate LF on $(x,y)$ as

\begin{equation}\label{eq:sv-neg}
{SV}_{p,w}^- = \frac{\frac{p}{p+w}-\frac{1}{C}-p*{SV}_{p,w}^+}{w}, \quad\quad  w>0
\end{equation}

We use $\varphi_{ij}^c$ to denote the Shapley value of $\lambda_j$ with respect to the MV label model in classifying $x_i$ when the hidden label of $x_i$ is $c$. The $\varphi_{ij}^c$ value can be computed directly as

\begin{equation} \label{eq:weshap-weights}
  \varphi_{ij}^c = \left\{
  \begin{array}{lr}
        SV_{p_i, w_i}^+, & \lambda_j(x_i)=c \\
        0, & \lambda_j(x_i)=\emptyset \\
        SV_{p_i, w_i}^-, & otherwise\\
  \end{array}
  \right.
\end{equation}

where $p_i$ and $w_i$ denote the number of correct and incorrect LFs in $\Lambda$ on $(x_i, y_i)$ respectively. We name the $\varphi_{ij}^c$ values \textbf{WeShap weights}, which correspond to the average contribution of $\lambda_j$ in classifying $x_i$ with respect to the MV model, when the hidden label of $x_i$ is $c$. 

\begin{example}
Figure \ref{fig:weshap} illustrates a running example, where we consider three LFs denoted as $\lambda_1$ through $\lambda_3$, each contributing to the labeling of data instances. The unlabeled training points, $x_1$ through $x_6$, are augmented with labeled validation points, $x_7$ and $x_8$, utilized for LF evaluation. 

In the bottom left of Figure \ref{fig:weshap}, we demonstrate the WeShap weights of the LFs on unlabeled training points. As a concrete example, we calculate the Weshap weight of $\lambda_1$ on $x_3$ when $y_3=0$. Notice that 1 correct LF ($\lambda_1$) and 1 wrong LF ($\lambda_2$) are activated on it. Following Equation \ref{eq:sv-pos} and \ref{eq:weshap-weights}, the Weshap weight of $\lambda_1$ on $x_3$ is $\frac{\psi(0,0)+\psi(0,1)}{2}=0.5$. The first part corresponds to the marginal utility of $\lambda_1$ in the permutation $[\lambda_1, \lambda_2]$, and the second part corresponds to that in $[\lambda_2, \lambda_1]$. The WeShap weight indicates that $\lambda_1$'s marginal contribution to the classification of $x_3$ is 0.5, averaging across all LF permutations under the MV model.  
\end{example}

Next, we focus on the downstream KNN model. Given a validation instance $(x_{val}, y_{val})$, let $\mathcal{N}_K(x_{val})$ denote the K-nearest neighbors of $x_{val}$ in the training dataset. The \textbf{WeShap value} of $\lambda_j$ on $(x_{val}, y_{val})$ is defined as:

\begin{equation}\label{eq:weshap-instance}
    \phi_j^{WeShap}(x_{val}, y_{val}) = \frac{1}{K}\sum_{i\in \mathcal{N}_K(x_{val})} \varphi_{ij}^{y_{val}}
\end{equation}

Similarly, the WeShap value of an LF for a holdout dataset $D_{val}$ is defined as the average WeShap value of the LF with respect to the instances inside the dataset:

\begin{equation}\label{eq:weshap-dataset}
\phi_j^{WeShap}(D_{val}) = \frac{\sum_{(x,y)\in D_{val}}\phi_j^{WeShap}(x, y)}{|D_{val}|}
\end{equation}

\begin{theorem}\label{theorem:weshap}
The WeShap value of an LF is equal to its Shapley value in the proxy game using the same set of LFs and holdout dataset.
\end{theorem}
\begin{proof}
Let's focus on a single validation instance $(x_{val}, y_{val})$ first. The utility of a coalition of LFs $S\subset \Lambda$ in the proxy game is:
\begin{equation}\label{eq:weshap-derivation}
\begin{aligned}
 v(S) &= f^{y_{val}}_{S}(x_{val}) - \frac{1}{C}  & \text{(Definition \ref{def:proxy-game})}\\
      &= \frac{1}{K}\left[\sum_{i\in \mathcal{N}_K(x_{val})}\Theta_S^{y_{val}}(x_i)\right]- \frac{1}{C} & \text{(Equation \ref{eq:knn-pred})}\\ 
      &= \frac{1}{K}\left[\sum_{i\in \mathcal{N}_K(x_{val})}\left(\Theta_S^{y_{val}}(x_i)-\frac{1}{C}\right)\right]& \\
 \end{aligned}
\end{equation}

Next, we define K coalition games $\{\mathcal{G}_i: i\in \mathcal{N}_K(x_{val})\}$. In each game, the LFs are used to classify one point $(x_i, y_{val})$ using the MV label model. The utility of each coalition game is defined as the label model's predictive accuracy on $(x_i, y_{val})$ minus $\frac{1}{C}$. 

Notice that $\Theta_S^{y_{val}}(x_i)-\frac{1}{C}$ in Equation \ref{eq:weshap-derivation} is the utility of $S$ in game $\mathcal{G}_i$. Following the additivity property of the Shapley value, the Shapley value of an LF $\lambda_j$ in the proxy game is the average of the Shapley values it receives in $\{\mathcal{G}_i: i\in \mathcal{N}_K(x_{val})\}$. Following Theorem \ref{theorem:mv}, the Shapley value of $\lambda_j$ receives in $\mathcal{G}_i$ is $\varphi_{ij}^{y_{val}}$. Therefore, the Shapley value of $\lambda_j$ in the proxy game is $\frac{1}{K}\sum_{i\in \mathcal{N}_K(x_{val})} \varphi_{ij}^{y_{val}}$, which is exactly the WeShap value. 

Since we have proved the WeShap value is the Shapley value of a LF in the proxy game when $D_{val}$ contains a single instance, following the additivity property of Shapley value and the definition in Equation \ref{eq:weshap-dataset}, we can conclude that the WeShap value is also the Shapley value of a LF when $D_{val}$ contains multiple instances.
\end{proof}

The WeShap value can also be generalized to KNN with non-uniform weights. To apply weighted KNN, we can simply modify Equation \ref{eq:weshap-instance} as:

\begin{equation}\label{eq:weshap-instance-weighted}
    \phi_j^{WeShap}(x_{val}, y_{val}) = \frac{\sum_{i\in N_K(x_{val})}\left(\omega(x_i, x_{val})* \varphi_{ij}^{y_{val}}\right)}{\sum_{i\in N_K(x_{val})}\omega(x_i, x_{val})}
\end{equation}

Where $w(x_i, x_{val})$ specify the weight given to $x_i$ when predicting $x_{val}$.

Finally, we can decompose the WeShap value in another way, enabling us to inspect each weak label's contribution. We define the \textbf{WeShap contribution scores} as:

\begin{equation} \label{eq:weshap-contribution}
   w_{ij} = \frac{\sum_{(x_{val},y_{val})\in D_{val}: i\in \mathcal{N}_K(x_{val})}\varphi_{ij}^{y_{val}}}{|D_{val}|*K}
\end{equation}

The WeShap contribution scores serve as detailed metrics quantifying the impact of weak labels $\lambda_j(x_i)$ on the accuracy of the proxy PWS pipeline. A higher contribution score suggests greater assistance to the pipeline's accuracy. 


\begin{example}
In Figure \ref{fig:weshap}, Following Equation \ref{eq:weshap-instance}, the WeShap score of $\lambda_1$ on $(x_7, y_7)$ (marked in red) is the average of $\{\varphi_{1,1}^0, \varphi_{2,1}^0, \varphi_{3,1}^0\}$, as $\{x_1, x_2, x_3\}$ are the KNNs of $x_7$ in the training data and $y_7=0$. The result is 0.5, indicating that $\lambda_1$ makes an average contribution of 0.5 in the proxy framework for predicting $x_7$. 

The WeShap Contribution Scores table at the bottom right of Figure \ref{fig:weshap} is the decomposition of WeShap values. We observe negative scores for $w_{32}$ and $w_{41}$, indicating detrimental effects on the downstream model's performance. $w_{32}$ is negative because $\lambda_2$ misclassify $x_3$ as class 1 and hurts the prediction of $x_7$ in $D_{val}$. Similarly, $w_{41}$ is negative because $\lambda_1$ misclassify $x_4$ as class 0 and hurts the prediction of $x_8$. 
Accordingly, we can enhance the PWS pipeline by excluding these weak labels (setting them to $\emptyset$). Following this adjustment, the downstream model achieves a perfect accuracy of 1.0, underscoring the effectiveness of WeShap contribution scores in refining the PWS pipeline.
\end{example}

\subsection{Computational Complexity}\label{sec:complexity}

\begin{figure}[htbp]
\centering
 \includegraphics[width=0.9\columnwidth]{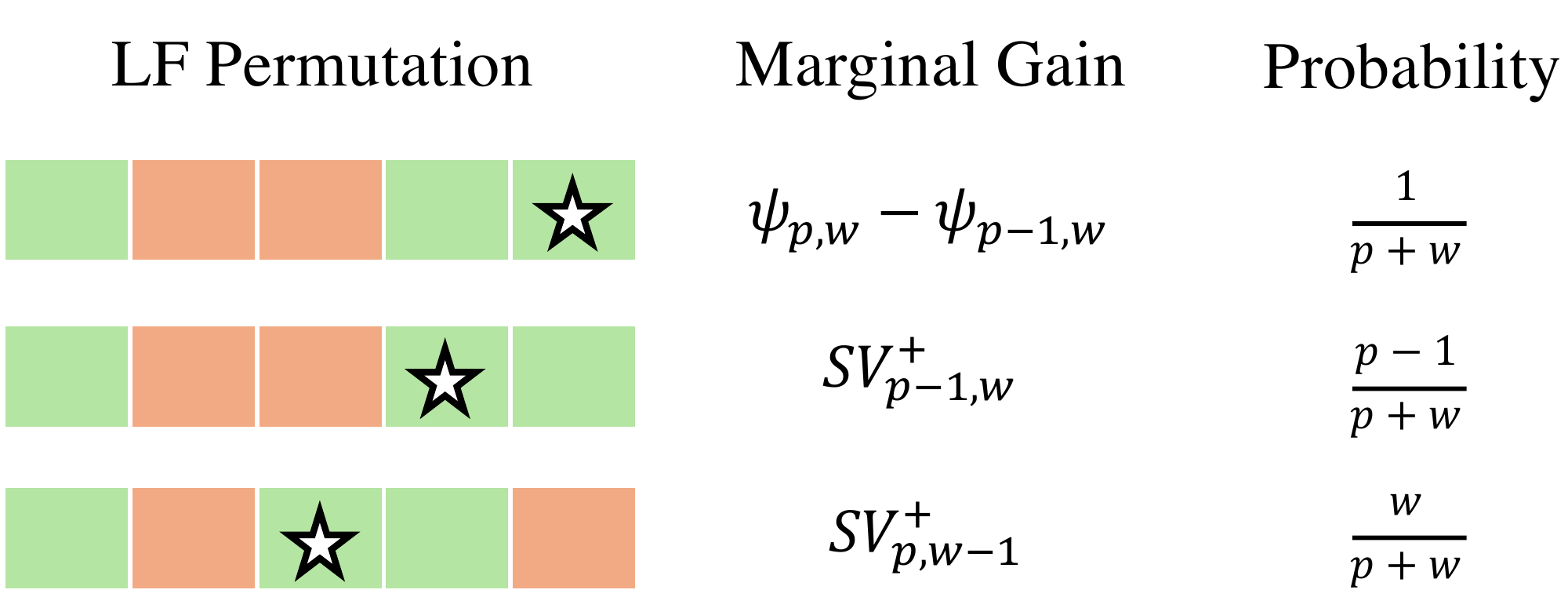}
 \caption{Efficient computation of WeShap scores using dynamic programming. The target LF $\lambda$ is denoted in star; green and red cells represent LFs making correct and wrong predictions, respectively. }\label{fig:dp-formula}
\end{figure}

The computational complexity of computing the WeShap value primarily stems from calculating the WeShap weights outlined in Equation \ref{eq:weshap-weights} and identifying the K-nearest neighbors for each validation instance. Brute-force KNN search entails a complexity of $\mathcal{O}(n_{val}dn)$ (where $d$ is the feature dimensionality), but this can be optimized to $\mathcal{O}(n_{val}d\log{n})$ by employing spatial data structures like KD-Trees \cite{bentley1975multidimensional}, with an additional $\mathcal{O}(n\log{n})$ complexity for constructing the KD-Tree.

Let's focus on computing the WeShap weights, which can be derived from $SV_{p,w}^+$ and $SV_{p,w}^-$ values. Directly computing all $SV_{p,w}^+$ values following Equation \ref{eq:sv-pos} requires $\mathcal{O}(m^4)$ computations, but this can be optimized further using dynamic programming. As illustrated in Figure \ref{fig:dp-formula}, we categorize all possible permutations of LFs into three cases:
\begin{itemize}
\item    The target LF $\lambda$ is the last in the permutation;
\item    Some other correct LF is the last in the permutation;
\item    Some incorrect LF is the last in the permutation.
\end{itemize}
In the first case, the marginal contribution of $\lambda$ is $\psi(p-1,w)$; in the second case, the average marginal contribution of $\lambda$ is $SV_{p-1,w}^+$; and in the last case, the average marginal contribution of $\lambda$ is $SV_{p, w-1}^+$. This yields the following recursive formula:

\begin{equation} \label{eq:recursive}
\begin{aligned}
SV_{p,w}^+&=\frac{\psi(p-1,w)}{p+w} + \frac{p-1}{p+w}SV_{p-1,w}^+ + \frac{w}{p+w}SV_{p,w-1}^+
\end{aligned}
\end{equation}

with the boundary values being
\begin{equation}
SV_{p,0}^+=\frac{C-1}{C*p}, \quad
SV_{0,w}^+= 0
\end{equation}

Leveraging the recursive formula, we can compute all $SV_{p,w}^+$ values in $\mathcal{O}(m^2)$ computations. The $SV_{p,w}^-$ values can be computed in $\mathcal{O}(m^2)$ computations using Equation \ref{eq:sv-neg}. These values are then mapped to the WeShap weights. Algorithm \ref{alg:weshap} demonstrates the algorithm to compute WeShap values using dynamic programming. The algorithm first utilizes Equation \ref{eq:recursive} to compute WeShap weights (Line 2-7). Then, it learns a KNN model to get the K nearest neighbors of each validation instance (Line 8) and aggregates WeShap weights to compute the WeShap contribution scores (Line 9-20). Finally, it aggregates WeShap contribution scores to get the WeShap values of each LF (Line 21-22). Assuming data structures like KD-Trees are applied for KNN search, the total computational complexity of Algorithm \ref{alg:weshap} is $\mathcal{O}(n\log{n}+n_{val} d\log{n}+Kn_{val}m+m^2)$, 
which generalizes well to large LF sets, improving significantly over the original $\mathcal{O}(2^m)$ time complexity for computing exact Shapley values.

\begin{algorithm}[tbp]
\DontPrintSemicolon
  
\KwInput{
$D_{train}$: unlabeled training dataset\newline
$D_{val}$: labeled holdout dataset \newline
$\Lambda$: labeling function set \newline
$C$: number of candidate classes \newline
$K$: number of neighbors in the proxy KNN model \newline
$\delta$: distance metric in the proxy KNN model \newline
$\omega$: weight function in the proxy KNN model 
}
\KwOutput{$\phi_j^{WeShap}$: WeShap values}

$n, m, n_{val} \leftarrow |D_{train}|, |\Lambda|, |D_{val}|$

$SV^+_{*,*}, w_{*,*} \leftarrow 0$

\For {$p \leftarrow 1$ to $m$}
{
    $SV^+_{p,0} \leftarrow \frac{C-1}{C*p}$
    
    \For {$w \leftarrow 1$ to $m$}
    {
        $SV^+_{p,w}\leftarrow \frac{\psi(p-1,w)}{p+w} + \frac{p-1}{p+w}SV_{p-1,w}^+ + \frac{w}{p+w}SV_{p,w-1}^+$

        $SV^-_{p,w} \leftarrow \frac{\frac{p}{p+w}-\frac{1}{C}-p*{SV}_{p,w}^+}{w}$
    } 
}

$f \leftarrow KNN(D_{train}, K, \delta)$

\For {$(x_{val}, y_{val}) \in D_{val}$}
{
    $\mathcal{N}_K (x_{val}) \leftarrow f.KNeighbors(x_{val})$

    $R \leftarrow n_{val} * \sum_{x_i \in \mathcal{N}_K (X_{val})} \omega(x_i, x_{val})$

    \For {$x_i \in \mathcal{N}_K (X_{val})$}
    {   
        $p_i \leftarrow \sum_{\lambda} \mathds{1}(\lambda(x_i)=y_{val})$

        $w_i \leftarrow \sum_{\lambda}\mathds{1}(\lambda(x_i)\neq \emptyset) - p_i$

        \For {$\lambda_j \in \Lambda$}
        {
            \eIf{$\lambda_j(x_i)=y_{val}$}
            {
                $w_{ij} \leftarrow w_{ij} + SV^+_{p_i,w_i}*\frac{\omega(x_i, x_{val})}{R}$
            }
            {
               \If{$\lambda_j(x_i)\neq \emptyset$}
               {
                    $w_{ij} \leftarrow w_{ij} + SV^-_{p_i,w_i}*\frac{\omega(x_i, x_{val})}{R}$
               }
            }
        }   
    }
}

\For {$j \leftarrow 1$ to $m$}
{
    $\phi_j^{WeShap} \leftarrow \sum_i w_{ij}$
}

\caption{Dynamic Programming Algorithm for WeShap Value Computation}
\label{alg:weshap}
\end{algorithm}

\subsection{Use Cases} \label{sec:use-cases}
The WeShap value proves valuable in several applications, elucidated in this section.


\textbf{Identify Helpful/Harmful LFs:} Higher WeShap values indicate greater contributions to the proxy pipeline, suggesting helpfulness, while negative values imply potential harm. This aids in filtering out detrimental LFs and optimizing resource allocation. For instance, if LFs stem from multiple supervision sources, users can allocate more resources to the sources yielding the most beneficial LFs.

\textbf{Enhance Downstream Model Accuracy:} We can refine LF outputs by silencing those with WeShap contribution scores below a threshold $\theta$:

\begin{equation}
\tilde{L}_{ij}=\left\{
\begin{array}{lr}
L_{ij} & w_{ij} \geq \theta \\
\emptyset & w_{ij} < \theta \\
\end{array}
\right.
\end{equation}
The threshold $\theta$ is tuned to optimize the PWS pipeline accuracy on the validation set. 

\textbf{Understanding PWS Pipeline Behaviors:}  WeShap scores identify LFs and training points with the highest or lowest contributions to validation instances, aiding in comprehending pipeline behaviors and diagnosing mispredictions. This is particularly useful when an LF indirectly influences predictions through nearby instances.

\textbf{Fair Credit Allocation:} Leveraging Shapley value properties (Section \ref{sec:shapley}), WeShap ensures fair attribution of credits to each LF, valuable when distributing payoffs among multiple contributors.

\section{Experiments} \label{sec:exp}

We have undertaken comprehensive experimentation to assess the efficacy of WeShap values across a spectrum of downstream tasks. These tasks encompass identifying beneficial LFs, enhancing the PWS pipeline's accuracy, and analyzing its behaviors. 

\subsection{Setup}\label{sec:exp-setup}
\begin{table*}[htbp]
\caption{Dataset summary statistics.}\label{tab:dataset}
\resizebox{2\columnwidth}{!}{%
\begin{tabular}{cc|cccc|ccccc}
\toprule[1.5pt]
\multirow{2}{*}{\textbf{Dataset}} & \multirow{2}{*}{\textbf{Task}} & \multicolumn{4}{c}{\textbf{Dataset Statistics}}                          & \multicolumn{5}{c}{\textbf{LF Statistics}}                                         \\ \cline{3-11} 
                                  &                                & \textbf{\#Class} & \textbf{\#Train} & \textbf{\#Valid} & \textbf{\#Test} & \textbf{\#LF} & \textbf{Acc} & \textbf{Cov} & \textbf{Overlap} & \textbf{Conflict} \\ \hline
YouTube                           & spam classification            & 2                & 1686             & 120              & 250             & 100           & 0.875        & 0.038        & 0.037            & 0.011             \\ \hline
IMDB                              & sentiment analysis             & 2                & 20000            & 2500             & 2500            & 100           & 0.666        & 0.026        & 0.024            & 0.017             \\ \hline
Yelp                              & sentiment analysis             & 2                & 30400            & 3800             & 3800            & 100           & 0.698        & 0.027        & 0.024            & 0.013             \\ \hline
TREC                              & question classification        & 6                & 5033             & 500              & 500             & 100           & 0.526        & 0.039        & 0.037            & 0.031             \\ \hline
MedAbs                            & disease classification         & 5                & 6002             & 3095             & 2657            & 100           & 0.444        & 0.025        & 0.023            & 0.018             \\ \hline
Mushroom                          & mushroom classification        & 2                & 6499             & 812              & 813             & 56            & 0.793        & 0.271        & 0.271            & 0.216             \\ \hline
Census                            & income classification          & 2                & 10083            & 5561             & 16281           & 83            & 0.787        & 0.054        & 0.053            & 0.015             \\ \hline
Indoor-Outdoor                    & image classification           & 2                & 1226             & 408              & 410             & 226           & 0.921        & 0.043        & 0.043            & 0.012             \\ \hline
VOC07-Animal                      & image classification           & 2                & 4008             & 1003             & 4952            & 296           & 0.950        & 0.042        & 0.042            & 0.012             \\ \bottomrule[1.5pt]
\end{tabular}}
\end{table*}

\paragraph{Datasets.} We evaluate our work extensively on 9 datasets, including 5 datasets for text classification (YouTube 
 \cite{alberto2015tubespam}, IMDB \cite{maas2011learning}, Yelp \cite{zhang2015character}, TREC \cite{li2002learning}, MedAbs \cite{schopf2022evaluating}), 2 datasets for tabular classification (Census \cite{misc_census_income_20}, Mushroom (MUSH) \cite{misc_mushroom_73}), and 2 datasets for image classification (Indoor-Outdoor (IND) \cite{tok2021practical}, and VOC07-Animal (VOC-A) \cite{pascal-voc-2007} \footnote{This is a binary classification version of the VOC-2007 dataset detecting whether an animal exist in the picture or not.}). These datasets have been widely used to evaluate PWS pipelines in prior works \cite{zhang2021wrench,zhang2022understanding,tok2021practical, hsieh2022nemo}.
 Table \ref{tab:dataset} outlines the dataset details and LF statistics. To assess LF quality in a scaled context, we generated LFs for the datasets using specific criteria:

For textual datasets, we employed LFs denoted as $\lambda_{k,c}$, which label class $c$ upon detecting a unigram $k$ in the text.

For tabular datasets, we utilized LFs denoted as $\lambda_{e,c}$, where class $c$ is assigned based on the truth value of expression $e$. We designed LFs for the Mushroom dataset and utilized the LF set introduced by \cite{awasthi2019learning} for the Census dataset. 

For image datasets, we employed the Azure Image Tagging API \footnote{\url{https://learn.microsoft.com/en-us/azure/ai-services/computer-vision/concept-tagging-images}(last accessed: 01/11/2024)} to associate tags with images corresponding to their visual features (e.g., sky, plant). Subsequently, we considered LFs denoted as $\lambda_{t,c}$, which assign class $c$ based on the existence of tag $t$.

We ensured LF quality by maintaining their accuracy at least 0.1 above random guessing on validation sets.

 \paragraph{PWS Pipeline.}
 We follow the standard PWS pipeline: train the label model on unlabeled data using LFs, exclude instances without active LFs, and then train the downstream model on remaining data with label model predictions. We evaluate pipeline performance via downstream model accuracy on the test set.

We assess two label models: majority voting and Snorkel MeTaL \cite{ratner2019training} implemented in WRENCH \cite{zhang2021wrench}. For downstream models, we consider two scenarios:

\textbf{(1) Feature extraction:} Use a frozen pretrained model to extract features, then train an end model on these features.

\textbf{(2) Fine-tuning:} Directly fine-tune the downstream model on weakly labeled data.

For feature extraction, we use dataset-specific feature extractors (Sentence-BERT \cite{reimers2019sentencebert} for YouTube, TREC, MedAbs; Bertweet-sentiment \cite{perez2021pysentimiento} for IMDB, Yelp; ResNet-50 \cite{he2016deep} for image datasets). We then train a logistic regression end model. For fine-tuning, we use BERT base \cite{devlin2018bert} for text and ResNet-50 \cite{he2016deep} for images. We set $batch\_size=32, n\_epochs=5, lr=5e-5, weight\_decay=0$ when fine-tuning BERT and $batch\_size=64, n\_epochs=50, lr=1e-4, weight\_decay=1e-5$ when fine-tuning the ResNet-50 model with an AdamW optimizer \cite{loshchilov2018decoupled} and apply early stopping technique to prevent overfitting. We don't fine-tune on tabular datasets lacking corresponding pretrained models. Each experiment is repeated five times with different seeds, reporting averaged results.

 \paragraph{Baselines. } We compare the following methods. 
 \begin{description}
    \item[Random (RND):] Assigns random values to LFs as a baseline.

    \item[Accuracy (ACC):] Evaluates LFs by validation set accuracy: $V(\lambda_j)=Acc(\lambda_j)$.

    \item[Coverage (COV):] Evaluates LFs by validation set coverage: $V(\lambda_j)=Cov(\lambda_j)$.

    \item [IWS \cite{boecking2020interactive}:]  Combines accuracy and coverage: $V(\lambda_j)=(2*Acc(\lambda_j)-1)*Cov(\lambda_j)$.
    This corresponds to evaluating the LF based on the correct prediction count minus the wrong prediction count on the validation set.

    \item [MC-Shap \cite{castro2009polynomial}:] Approximates Shapley value using Monte Carlo permutation sampling, evaluating gain by end model accuracy on validation set. To balance approximation accuracy and runtime cost, we use $n=100$ samples in our evaluation.

    \item[SIF \cite{zhang2022understanding}:] Learns fine-grained source-aware influence functions computed on the validation set: $V(\lambda_j)=|\sum_{i,c} \bar{\phi}_{i,j,c}|$, where $\bar{\phi}_{i,j,c}$ is the source-aware influence function of $\lambda_j$ on $x_i$ with respect to class c.
    


    \item[WeShap:] Our proposed method using Shapley values. We optimize K (5-40), distance metric, and weight function for each dataset based on proxy KNN classifier accuracy on the validation set.
    
 \end{description}
 
 \subsection{LF Evaluation}
\begin{table*}[htbp]
\caption{Average downstream model accuracy after ranking LFs based on different metrics.}
\label{tab:lf-rank}
\resizebox{2\columnwidth}{!}{%
\begin{tabular}{cccccccccccc}
\toprule[1.5pt]
LM & Metric & Youtube & IMDB & Yelp & MedAbs & TREC & MUSH & Census & IND & VOC-A & AVG \\ \hline
\multirow{7}{*}{MV} & RND & 0.825 & 0.831 & 0.877 & 0.497 & 0.536 & 0.876 & 0.782 & 0.889 & 0.934 & 0.783 \\ \cline{2-12} 
 & ACC & 0.813 & \textbf{0.842} & 0.888 & 0.489 & 0.583 & 0.840 & 0.809 & 0.884 & 0.940 & 0.787 \\ \cline{2-12} 
 & COV & 0.834 & 0.838 & 0.849 & 0.498 & 0.574 & 0.834 & 0.798 & 0.885 & 0.937 & 0.783 \\ \cline{2-12} 
 & IWS & 0.789 & 0.840 & 0.856 & 0.485 & \textbf{0.605} & \textbf{0.900} & \textbf{0.818} & 0.885 & 0.946 & 0.792 \\ \cline{2-12} 
 & MC-Shap & 0.818 & 0.837 & 0.858 & \textbf{0.556} & 0.539 & 0.895 & 0.813 & \textbf{0.910} & 0.930 & 0.795 \\ \cline{2-12} 
 & SIF & \textbf{0.854} & 0.838 & 0.872 & 0.507 & 0.587 & 0.885 & 0.802 & 0.888 & 0.951 & 0.798 \\ \cline{2-12} 
 & WeShap & 0.844 & 0.836 & \textbf{0.900} & 0.522 & 0.566 & 0.893 & \textbf{0.818} & 0.901 & \textbf{0.952} & \textbf{0.803} \\ \hline
\multirow{7}{*}{Snorkel} & RND & 0.735 & 0.778 & 0.690 & 0.434 & 0.475 & 0.835 & 0.781 & 0.879 & 0.882 & 0.721 \\ \cline{2-12} 
 & ACC & 0.667 & 0.796 & 0.722 & 0.445 & 0.504 & 0.864 & 0.778 & 0.877 & 0.922 & 0.730 \\ \cline{2-12} 
 & COV & 0.685 & 0.814 & 0.621 & 0.452 & 0.479 & 0.829 & 0.763 & 0.883 & 0.883 & 0.712 \\ \cline{2-12} 
 & IWS & 0.665 & \textbf{0.821} & 0.629 & 0.441 & \textbf{0.544} & \textbf{0.892} & 0.758 & 0.886 & 0.893 & 0.725 \\ \cline{2-12} 
 & MC-Shap & 0.811 & 0.782 & 0.743 & 0.466 & 0.480 & 0.886 & 0.778 & 0.891 & \textbf{0.948} & 0.754 \\ \cline{2-12} 
 & SIF & 0.727 & 0.793 & 0.741 & 0.440 & 0.460 & 0.882 & \textbf{0.793} & 0.883 & 0.869 & 0.732 \\ \cline{2-12} 
 & WeShap & \textbf{0.827} & 0.810 & \textbf{0.836} & \textbf{0.470} & 0.473 & 0.883 & 0.775 & \textbf{0.897} & 0.938 & \textbf{0.768} \\ \bottomrule[1.5pt]
\end{tabular}%
}
\end{table*}

In assessing LF quality, we confront the challenge posed by the varying scales of LF evaluation metrics. To standardize this process, we implement the following methodology: initially, we arrange LFs in descending order based on their evaluation metrics, mitigating the influence of scale discrepancies. Subsequently, we adopt an iterative approach where we progressively select the top-p LFs and utilize them to train both the label and downstream models. Our evaluation commences with the top-10 LFs and incrementally expands the LF subset size at intervals of 10 until it encompasses all LFs within the dataset. We present the average accuracy of the downstream model on the test set across these iterations in Table \ref{tab:lf-rank}. As this evaluation process requires training the downstream model repeatedly, which is time-consuming for the fine-tuning scenario, we only evaluated the feature-extraction scenario for the LF ranking experiments. 


Table \ref{tab:lf-rank} illustrates the superiority of WeShap values in ranking beneficial LFs. While WeShap does not consistently yield the optimal result, it demonstrates robust performance across diverse datasets. For both label model choices, WeShap achieves the highest average downstream model accuracy. Notably, WeShap exhibits particular efficacy when using Snorkel as the label model — a common PWS configuration in recent studies \cite{zhang2021wrench,hsieh2022nemo,zhang2022understanding}. In this scenario, WeShap outperforms the second-best metric (MC-Shap) by 1.4 points and the third-best metric (SIF) by 3.6 points on average while significantly reducing runtime, as demonstrated in Figure \ref{fig:runtime}. The most substantial improvement is observed in the Yelp dataset, where WeShap shows a remarkable 9.3-point increase over the next best metric.



\begin{figure}[htbp]
\centering
 \includegraphics[width=0.9\columnwidth]{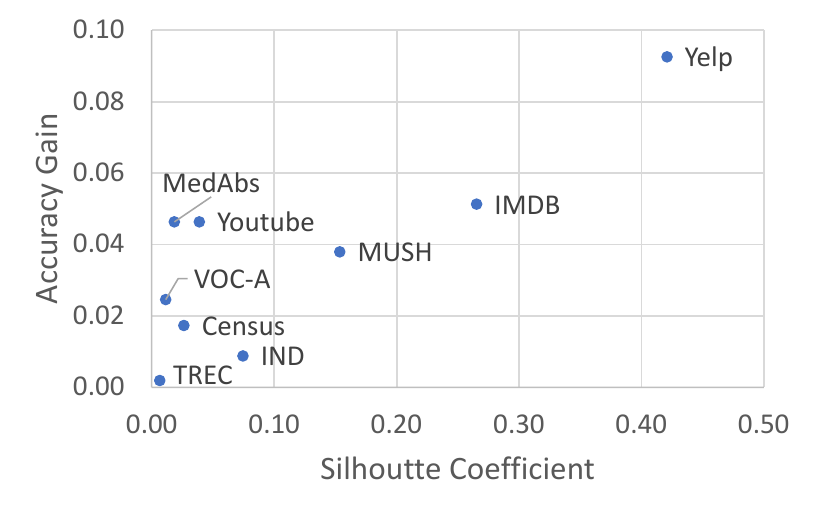}
 \caption{Average downstream model accuracy gain of WeShap compared to random baseline in ranking LFs.}\label{fig:acc-gain}
\end{figure}

\begin{figure*}
     \centering
     \begin{subfigure}[b]{0.33\textwidth}
         \centering
         \includegraphics[width=\textwidth]{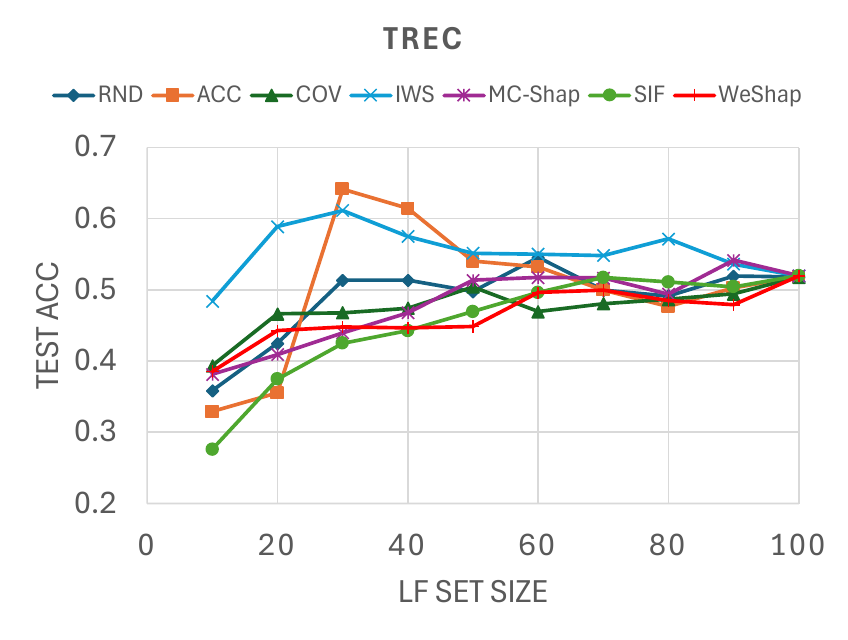}
     \end{subfigure}
     \hfill
     \begin{subfigure}[b]{0.33\textwidth}
         \centering
         \includegraphics[width=\textwidth]{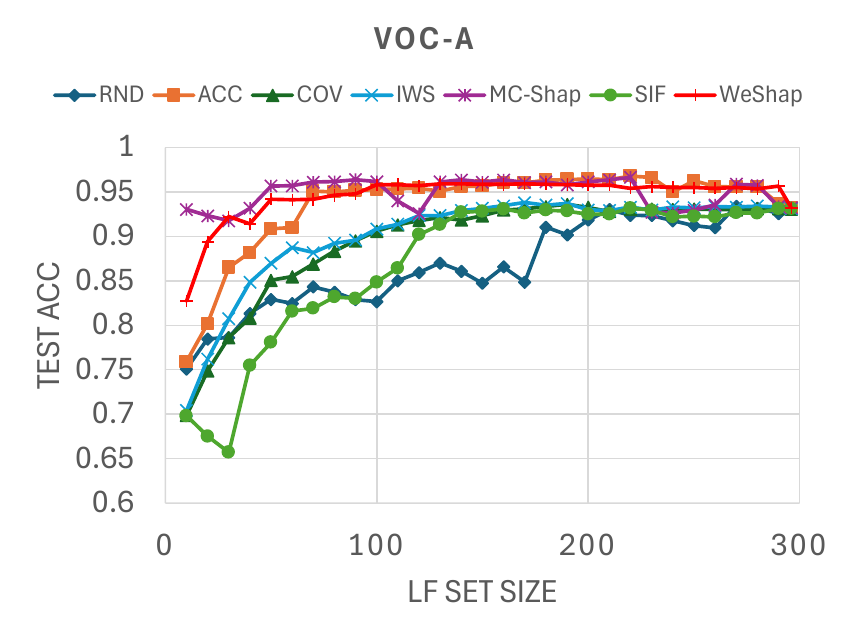}
     \end{subfigure}
     \hfill
     \begin{subfigure}[b]{0.33\textwidth}
         \centering
         \includegraphics[width=\textwidth]{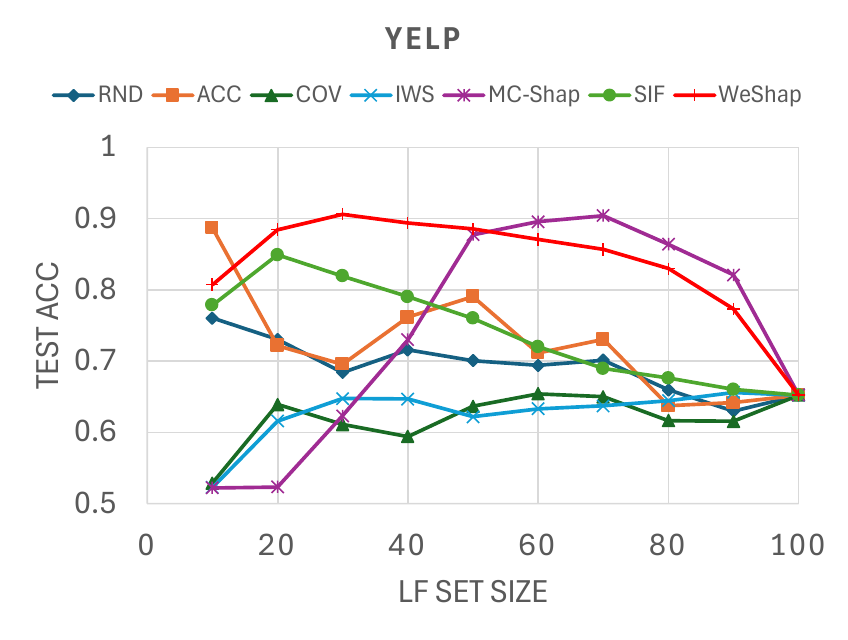}
     \end{subfigure}
        \caption{Ranking LFs on selected datasets.}
        \label{fig:lf-rank-selected}
\end{figure*}

To discern scenarios where WeShap values reliably rank LFs, we analyze the data distribution across the evaluated datasets and the characteristics associated with LFs. Notably, we observe a strong correlation between WeShap's LF ranking performance and the smoothness of the training data.

The smoothness assumption, a cornerstone in semi-supervised learning, posits that points in a high-dimensional space proximal to each other should share similar labels \cite{van2020survey}. To quantify data smoothness, we employ the silhouette coefficient \cite{rousseeuw1987silhouettes}, where a higher coefficient denotes superior smoothness. Given the varying factors affecting absolute end model accuracy across datasets, we assess WeShap's efficacy by measuring its average downstream model accuracy gain relative to a random baseline. 

Figure \ref{fig:acc-gain} depicts the silhouette coefficient plotted against the downstream model accuracy gain in the proxy framework. Evidently, a positive correlation ($r=0.810$) emerges between dataset smoothness and WeShap's accuracy gain. The rationale lies in WeShap's utilization of KNN as the downstream model within the proxy pipeline, wherein the KNN model's performance directly hinges on data smoothness. Consequently, if the dataset exhibits poor smoothness, the KNN model's accuracy diminishes, thereby impairing its efficacy as a proxy model for identifying beneficial LFs. 
The low smoothness in the TREC dataset likely contributes to WeShap's poor performance, a finding corroborated by the inferior performance of the Sentence-BERT encoder on TREC, as reported in \cite{reimers2019sentencebert}. While the encoder's poor performance affects all evaluated metrics, it has a more pronounced effect on WeShap, which uniquely relies on the smoothness assumption, as previously discussed.

To better illustrate the effect of LF set size on downstream model performance, Figure \ref{fig:lf-rank-selected} depicts the downstream model accuracy as we incorporate more LFs in the selected set across three datasets. As the dataset's smoothness level increases, WeShap demonstrates greater performance gains than baseline methods. In the TREC dataset, WeShap's performance is comparable to the random baseline, whereas in the Yelp dataset, it significantly outperforms all other metrics. Notably, WeShap exhibits performance advantages from the early stages in the VOC-Animal and Yelp datasets, with as few as tens of LFs selected. This underscores WeShap's utility in identifying the most beneficial LFs from a large set. Another observation is that when the metric reliably evaluates LF helpfulness (e.g., WeShap on Yelp), the downstream model accuracy resembles an inverted U-shape as the number of LFs increases. This occurs because we first include helpful LFs, followed by harmful ones in the selected LF set. The position of the curve's peak is dataset-dependent. Conversely, if the metric distinguishes between helpful and harmful LFs poorly, the downstream model accuracy tends to change monotonically as more LFs are incorporated into the set.

\subsection{PWS Pipeline Revision} \label{sec:pws-revision}
\begin{table*}[]
\caption{Downstream model accuracy after revising the PWS pipeline. Tabular datasets (MUSH, Census) do not have correspond-
ing pre-trained models for fine-tuning. }
\label{tab:revision}
\resizebox{2\columnwidth}{!}{%
\begin{tabular}{cccccccccccc}
\toprule[1.5pt]
Scenario & Metric & YouTube & IMDB & Yelp & MedAbs & TREC & MUSH & Census & IND & VOC-A & AVG \\ \hline
\multirow{9}{*}{\begin{tabular}[c]{@{}c@{}}Feature\\ Extraction\\ (LogReg)\end{tabular}} & Base & 0.744 & 0.811 & 0.655 & 0.433 & 0.487 & 0.906 & 0.751 & 0.886 & 0.932 & 0.734 \\ \cline{2-12} 
 & ACC & 0.663 & 0.845 & 0.839 & 0.468 & 0.613 & 0.957 & 0.769 & 0.909 & 0.949 & 0.779 \\ \cline{2-12} 
 & COV & 0.718 & 0.809 & 0.638 & 0.445 & 0.496 & 0.900 & 0.764 & 0.898 & 0.786 & 0.717 \\ \cline{2-12} 
 & IWS & 0.712 & 0.840 & 0.757 & 0.477 & 0.621 & 0.862 & 0.764 & 0.901 & 0.808 & 0.749 \\ \cline{2-12} 
 & MC-Shap & 0.841 & 0.808 & 0.911 & 0.479 & 0.517 & 0.927 & 0.788 & 0.904 & \textbf{0.959} & 0.793 \\ \cline{2-12}
 & SIF-P & 0.747 & 0.807 & 0.857 & 0.475 & 0.526 & 0.890 & 0.790 & 0.880 & 0.810 & 0.754 \\ \cline{2-12} 
 & SIF-F & \textbf{0.917} & 0.844 & 0.895 & 0.430 & 0.549 & 0.914 & 0.781 & \textbf{0.911} & 0.925 & 0.796 \\ \cline{2-12} 
 & WeShap-P & 0.868 & 0.827 & 0.905 & 0.486 & 0.496 & 0.927 & 0.764 & 0.905 & 0.950 & 0.792 \\ \cline{2-12} 
 & WeShap-F & 0.910 & \textbf{0.852} & \textbf{0.927} & \textbf{0.603} & \textbf{0.632} & \textbf{0.990} & \textbf{0.838} & 0.909 & 0.952 & \textbf{0.846} \\ \cline{2-12} 
 & \textit{Golden} & \textit{0.914} & \textit{0.873} & \textit{0.942} & \textit{0.646} & \textit{0.922} & \textit{1.000} & \textit{0.807} & \textit{0.929} & \textit{0.972} & \textit{0.889} \\ \hline
\multirow{9}{*}{\begin{tabular}[c]{@{}c@{}}Finetuning\\ (BERT/\\ Resnet-50)\end{tabular}} & Base & 0.877 & 0.789 & 0.770 & 0.445 & 0.495 & -- & -- & 0.880 & 0.855 & 0.730 \\ \cline{2-12} 
 & ACC & 0.856 & 0.812 & 0.830 & 0.409 & 0.590 & -- & -- & 0.905 & 0.977 & 0.768 \\ \cline{2-12} 
 & COV & 0.890 & 0.803 & 0.684 & 0.482 & 0.466 & -- & -- & 0.885 & 0.778 & 0.712 \\ \cline{2-12} 
 & IWS & 0.890 & 0.797 & 0.635 & 0.501 & 0.516 & -- & -- & 0.886 & 0.926 & 0.736 \\ \cline{2-12} 
 & SIF-P & 0.881 & 0.806 & 0.827 & 0.466 & 0.454 & -- & -- & 0.887 & 0.715 & 0.719 \\ \cline{2-12} 
 & SIF-F & \textbf{0.934} & 0.857 & 0.866 & 0.489 & 0.517 & -- & -- & \textbf{0.914} & 0.946 & 0.789 \\ \cline{2-12} 
 & WeShap-P & 0.849 & 0.822 & 0.872 & 0.473 & 0.495 & -- & -- & 0.877 & 0.920 & 0.758 \\ \cline{2-12} 
 & WeShap-F & 0.919 & \textbf{0.882} & \textbf{0.942} & \textbf{0.539} & \textbf{0.621} & -- & -- & \textbf{0.914} & \textbf{0.978} & \textbf{0.828} \\ \cline{2-12} 
 & \textit{Golden} & \textit{0.968} & \textit{0.888} & \textit{0.955} & \textit{0.644} & \textit{0.960} & \textit{--} & \textit{--} & \textit{0.918} & \textit{0.981} & \textit{0.902} \\ \bottomrule[1.5pt]
\end{tabular}%
}
\end{table*}

Subsequently, we assess the impact of WeShap scores on refining the PWS pipeline regarding the downstream model's accuracy post-refinement. We explore two distinct refinement strategies: (1) Pruning: Eliminating detrimental LFs with low evaluation scores, and (2) Fine-Grained Revision: Modifying specific LF outputs or soft labels predicted by the label model.

Both WeShap and SIF support fine-grained revision, thus enabling the evaluation of both options for these methods. We designate the approaches as follows: WeShap-P (or SIF-P) for pruning out harmful LFs and WeShap-F (or SIF-F) for fine-grained revision. As for other metrics, we solely evaluate their efficacy in pruning out harmful LFs, as they lack support for fine-grained revision.

The revision pipeline operates as follows:
For LF pruning, we commence by ranking the LFs based on the evaluation metric, subsequently selecting the top-p LFs for training the label model. We tune the threshold p to optimize the PWS pipeline's accuracy on the validation dataset utilizing Optuna \cite{optuna_2019} with a Tree-structured Parzen Estimator (TPE) Sampler for 20 trials.

For fine-grained revision, we adjust threshold $\theta$ for muting LF outputs (as described in Section \ref{sec:use-cases}) for WeShap-F, and threshold $\alpha$ for perturbing train losses for SIF-F (which serves a similar purpose as $\theta$ for WeShap, as elaborated in \cite{zhang2022understanding}). We optimize these thresholds to maximize the PWS pipeline's accuracy on the validation dataset using Optuna with a TPE Sampler for 100 trials. 

Additionally, we include two baselines for comparison: the \textit{Base} method reflects the downstream model's accuracy without any revision, while the \textit{Golden} method showcases the downstream model's accuracy when trained on golden training labels without weak supervision. These baselines provide insight into the performance of specific revision methods and their disparity compared to using golden labels.

Table \ref{tab:revision} presents various revision methods' performance under feature extraction and fine-tuning scenarios. 
We evaluate the MC-Shap metric exclusively in the feature extraction scenario, as computing MC-Shap in fine-tuning scenarios is prohibitively time-consuming. Note that the MUSH and Census datasets are tabular ones that do not have corresponding pre-trained models for fine-tuning, so we omit the fine-tuning results on these two datasets. Due to space limitations, we report the figures when Snorkel is used as the label model. The performance comparison results are similar when using the MV model. 

We first focus on the feature-extraction scenario. In the context of LF pruning, both the WeShap and MC-Shap methods emerge as frontrunners, showcasing an advantage of 1.3 points over alternative pruning strategies. The result indicates the effectiveness of the Shapley value in ranking LFs. When integrating fine-grained supervision, the fine-grained revision methodologies surpass traditional LF pruning techniques. Particularly, WeShap-F distinguishes itself by outshining all other baseline methods in 6 out of 9 datasets, with the MedAbs dataset witnessing the most substantial leap in performance, an 11.7-point increase. 
On average, across nine datasets, WeShap-F significantly boosts downstream model accuracy by 11.2 points over using original LFs and exceeds the performance of state-of-the-art (SOTA) revision techniques by 5.0 points. For the fine-tuning scenario, we observe similar trends, where WeShap outperforms other baseline methods in 6 out of 7 datasets and achieves competitive results in the remaining one, surpassing the performance of SOTA revision techniques by 3.9 points on average. 

Surprisingly, WeShap demonstrates superior performance when revising the PWS pipeline on the TREC dataset despite its mediocre ranking of LFs on the same dataset. We hypothesize that this discrepancy arises because fine-grained revision performance is more influenced by local data smoothness (where neighboring data points share similar labels) than global data smoothness (where data points with similar labels cluster together). To test this hypothesis, we trained a KNN model on TREC using ground-truth labels, achieving an accuracy of 73.2\%. This result suggests that TREC exhibits good local smoothness despite its poor global smoothness, as indicated by a low Silhouette score. 

\subsection{Understanding Pipeline Behaviors}

\begin{table}[]
\caption{LF statistics for VOC07-Animal (Snorkel-LogReg).}
\label{tab:lf-voc07}
\resizebox{\columnwidth}{!}{%
\begin{tabular}{cccccc}
\toprule[1.5pt]
LF                           & Accuracy & Coverage & Overlap & Conflict & WeShap \\ \hline
animal$\to$Y       & 0.992    & 0.223    & 0.223   & 0.221    & 0.024  \\ \hline
mammal$\to$Y       & 0.997    & 0.154    & 0.154   & 0.152    & 0.015  \\ \hline
vehicle$\to$N      & 0.977    & 0.342    & 0.342   & 0.027    & 0.012  \\ \hline
...                          & ...      & ...      & ...     & ...      & ...    \\ \hline
tree$\to$N        & 0.626    & 0.161    & 0.161   & 0.068    & -0.007 \\ \hline
outdoor$\to$N      & 0.660    & 0.604    & 0.604   & 0.229    & -0.022 \\ \bottomrule[1.5pt]
\end{tabular}%
}
\end{table}

We present a case study on the VOC-Animal dataset to demonstrate how WeShap values can be used to understand and improve the PWS pipeline through human-in-the-loop intervention. The study uses Snorkel as the label model, Resnet-50 as the feature extractor, and logistic regression as the downstream model. Class Y indicates the presence of animals, while N indicates their absence.

Initially, the test set accuracy was 0.920. Table \ref{tab:lf-voc07} shows LF statistics for the VOC07-Animal dataset, ranked by WeShap values. The top LF, assigning Y when the image is tagged "animal," has a positive WeShap value of 0.024, indicating it improves the proxy downstream model (KNN) accuracy by an average of 0.024. The second LF, "mammal$\to$Y," also has a positive WeShap value of 0.015. Based on these findings, we added three new animal-related LFs ("cat$\to$Y," "dog$\to$Y," and "bird$\to$Y"), increasing test set accuracy to 0.934. Following that, we observed that the LF "outdoor$\to$N" had a negative WeShap value of -0.022 despite an accuracy of 0.66. Removing this LF improved accuracy to 0.941. Lastly, noting that high WeShap values corresponded to high-accuracy LFs, we increased the LF selection threshold to 0.7, further improving test set accuracy to 0.957. These revisions led to a total accuracy improvement of 0.034, a substantial enhancement given that the downstream model's accuracy using ground truth labels is 0.972.

WeShap values also offer a versatile application in identifying the most influential LFs and training data pertinent to specific test instances. Illustrated in Figure \ref{fig:misprediction}, we showcase a subset of images from the validation dataset that the downstream model mispredicts. Subsequently, we compute WeShap scores for these mispredicted images individually.

The influential LFs are identified as those with the lowest WeShap scores, while the influential training images are determined by the images associated with the lowest WeShap contribution scores. Our analysis reveals that the mispredictions predominantly stem from certain LFs associated with the absence of animals. Consequently, users can opt to discard or downweight these LFs to rectify these mispredictions.


\begin{figure*}[htbp]
\centering
 \includegraphics[width=1.8\columnwidth]{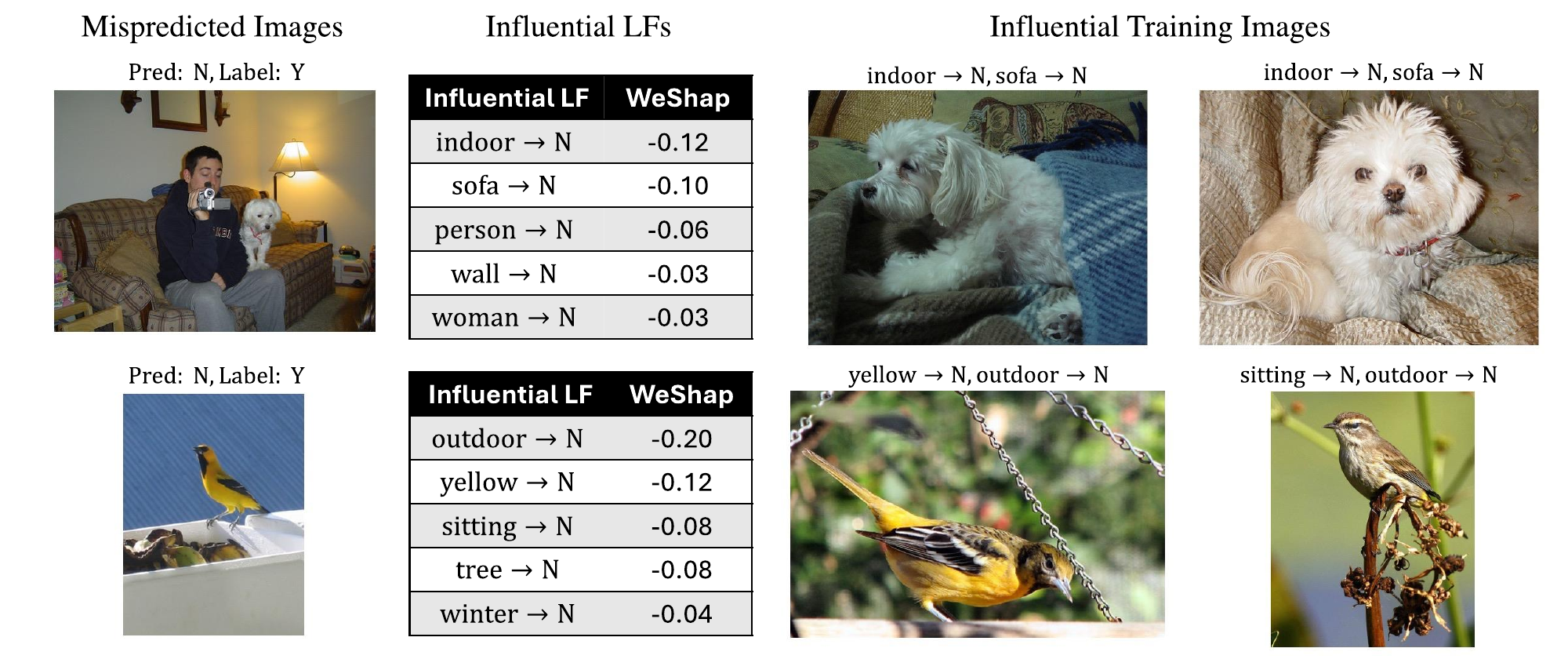}
 \caption{Analyze mispredictions in VOC07-Animal dataset (Snorkel-LogReg). Class "Y" denotes the presence of animals, and "N" denotes the absence of them. }\label{fig:misprediction}
\end{figure*}

\subsection{Ablation Studies}

\begin{figure*}[htbp]
\centering
\begin{subfigure}[c]{\columnwidth}
 \centering
 \includegraphics[width=0.9\columnwidth]{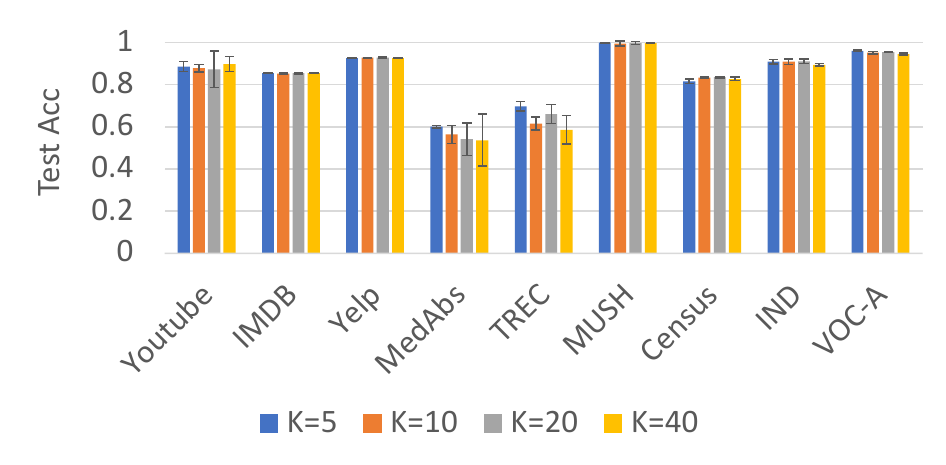}
 \caption{Effect of K values.}\label{fig:ablation_k}
\end{subfigure}
\hfill
\begin{subfigure}[c]{\columnwidth}
\centering
 \includegraphics[width=0.9\columnwidth]{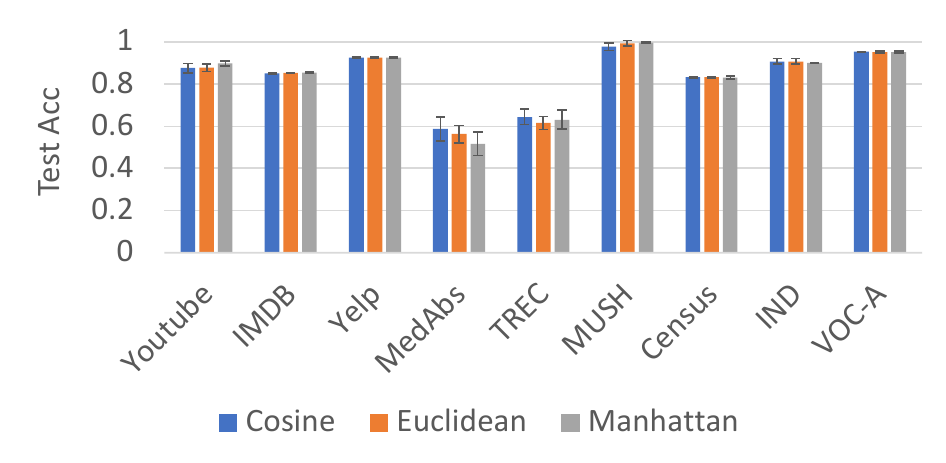}
 \caption{Effect of distance metrics.}\label{fig:ablation_distance}
\end{subfigure}

\begin{subfigure}[c]{\columnwidth}
\centering
 \includegraphics[width=0.9\columnwidth]{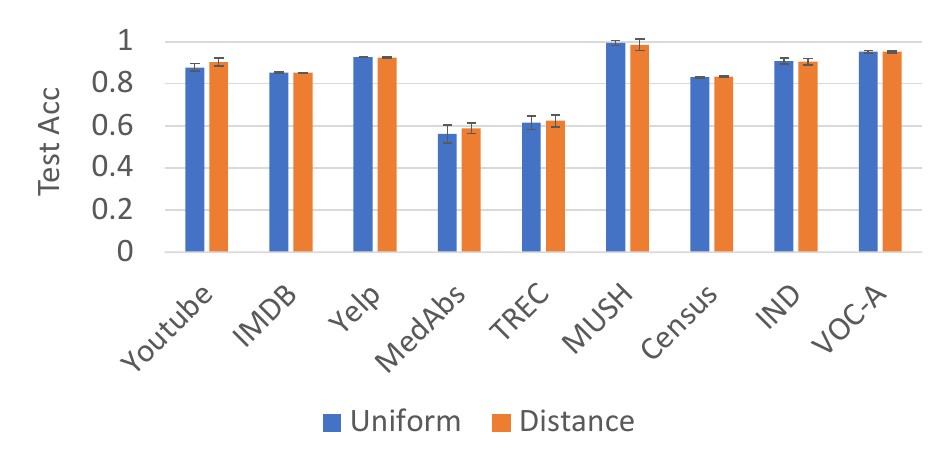}
 \caption{Effect of weight functions.}\label{fig:ablation_weight}
\end{subfigure}
\hfill
\begin{subfigure}[c]{\columnwidth}
\centering
 \includegraphics[width=0.9\columnwidth]{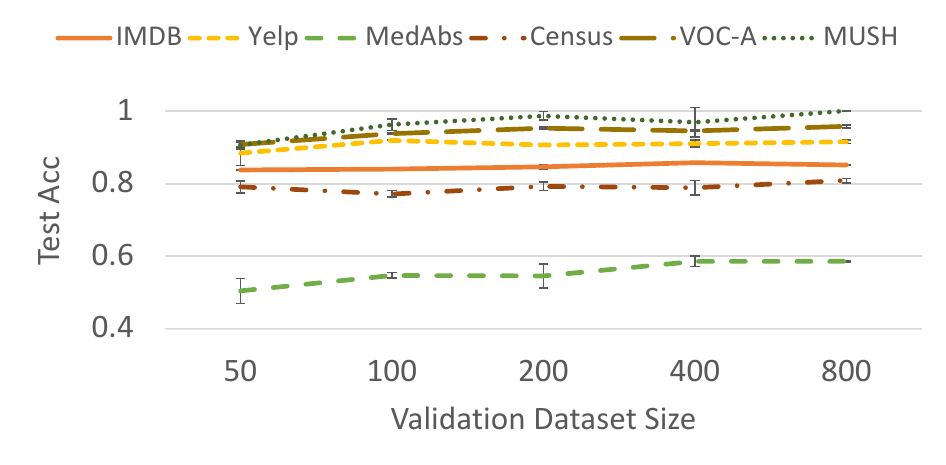}
 \caption{Effect of holdout dataset size.}\label{fig:ablation_valid}
\end{subfigure}
\caption{Effect of WeShap configurations for PWS pipeline revision (Snorkel-Logreg).}\label{fig:ablation}
\end{figure*}

\begin{figure}[htbp]
\centering
 \includegraphics[width=0.9\columnwidth]{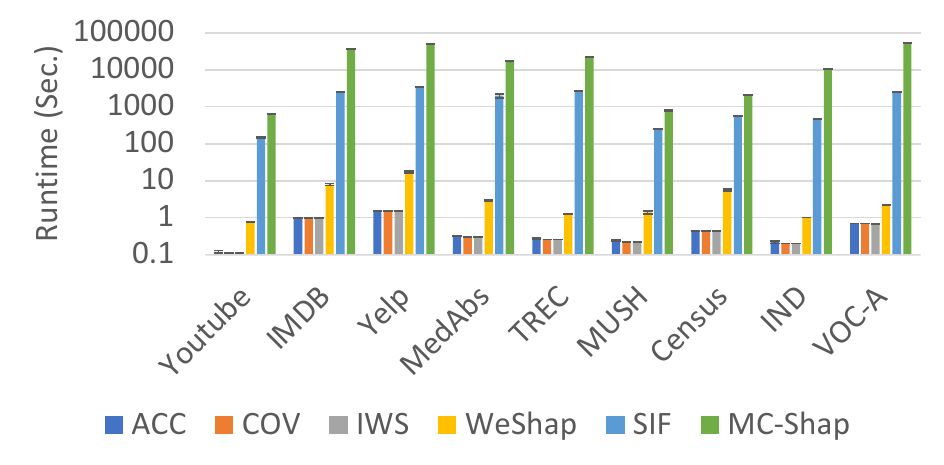}
 \caption{Runtime for different LF evaluation metrics.}\label{fig:runtime}
\end{figure}

We assess the sensitivity of WeShap scores to various configuration choices, including the number of neighbors (K), distance metrics, weight functions, and holdout dataset size. Our default settings are K=10, Euclidean distance, and uniform weight. We evaluate performance using downstream model accuracy after fine-grained PWS pipeline revision (WeShap-F). We test the following configuration choices:

\begin{description}
  \item[Number of neighbors (K):] 5, 10, 20, and 40.
  \item[Distance metrics:] Cosine, Euclidean, and Manhattan.
  \item[Weight functions:] Uniform (Equation \ref{eq:weshap-instance}) and inverse distance (Equation \ref{eq:weshap-instance-weighted}).
  \item[Holdout dataset sizes:] 50, 100, 200, 400, and 800 (for datasets with over 800 validation points).
\end{description}

The findings, as depicted in Figure \ref{fig:ablation}, demonstrate the robustness of WeShap scores in enhancing the PWS pipeline, irrespective of variations in K values, distance metrics, and weight functions. Moreover, enlarging the holdout dataset consistently enhances the performance of the PWS pipeline. Specifically, doubling the holdout dataset size correlates with a 1.2-point increase in downstream model accuracy across our experiments. Remarkably, even with as few as 50 labeled data points, WeShap scores contribute to a substantial 5.7-point average improvement in downstream model accuracy compared to the original PWS pipeline. These results underscore the continued utility of WeShap scores, particularly in scenarios with constrained labeling budgets.






\subsection{Runtime}
The experiments were conducted on a 4-core Intel(R) Xeon(R) Gold 5122 CPU (3.60 GHz) with an NVIDIA TITAN Xp GPU (12 GB memory). Multiprocessing and GPU acceleration were used for MC-Shap and SIF score computations, respectively. Figure \ref{fig:runtime} shows the runtimes for LF evaluation metrics, excluding label model and end model training times. Results are averaged over five runs, with error bars indicating standard deviation.

WeShap scores were computed within seconds across all datasets, while SIF and MC-Shap required significantly longer — often hours to days, even with parallelization. On average, SIF and MC-Shap took 356 and 5,300 times longer than WeShap, respectively. This difference is due to their computational complexity: SIF calculates Hessian-vector products recursively, while MC-Shap requires training both label and downstream models $m \times n$ times (where $m$ is the LF set cardinality and $n$ is the Monte Carlo sample size). In our evaluation, this resulted in tens of thousands of model training iterations. Consequently, WeShap offers a more efficient and lightweight approach to LF evaluation.

\section{Discussions}
In this section, we compare WeShap with selected baseline methods in our evaluation, highlighting its key advantages and limitations. 

\textbf{Weshap vs SIF.} Both methods offer fine-grained influence decomposition but differ in their approach.
SIF uses the \textit{(i-j-c) effect} to quantify the weight of the label model for $L_{ij}$ in predicting $P(y_i=c)$. WeShap employs \textit{WeShap weights} $\varphi_{ij}^c$ to measure the \textit{average contribution} of $L_{ij}$ in predicting $P(y_i=c)$. The key distinction is that while SIF focuses on current label model weights, WeShap considers all possible LF subsets to compute average contributions.

For example, when assessing $\lambda_j$'s impact on $x_i$ in the MV label model, the (i-j-c) effect is binary (1 when $L_{ij}=c$, 0 otherwise). In contrast, the WeShap weight $\varphi_{ij}^c$ accounts for other LFs' outputs on $x_i$, providing a more comprehensive view of LF contribution.

Our evaluation shows that WeShap outperforms SIF in ranking LFs and improving PWS pipeline accuracy while being computationally more efficient. While approaches like FastIF \cite{guo2021fastif} could potentially be applied to accelerate SIF, they are unlikely to bridge the runtime gap with WeShap fully. Optimizing influence score computation remains an open research area, though beyond our current scope. In summary, WeShap's broader reflection of LF contributions and efficiency give it practical advantages over SIF, despite both offering valuable insights into LF influence within the PWS pipeline.

\textbf{WeShap vs MC-Shap.} Both WeShap and MC-Shap are founded on Shapley values. MC-Shap's primary advantage lies in its ability to use the exact label and end models for Shapley value computation, potentially yielding more accurate LF evaluation results compared to proxy models. However, the Monte Carlo sampling process introduces estimation errors that may offset this advantage. As our demonstrations have shown, MC-Shap values are computationally expensive, rendering it impractical to evaluate a large number of samples within a reasonable timeframe. Furthermore, unlike WeShap, MC-Shap does not support fine-grained decomposition or revision, leading to inferior performance in PWS revision tasks.

\textbf{Limitations.} We acknowledge several limitations of WeShap values in this study. Firstly, our theoretical analysis is predicated on a specific proxy PWS setting, as discussed in Section \ref{sec:proxy}. Consequently, the WeShap values may not accurately reflect the ground-truth Shapley values of LFs in alternative settings. Extending the theoretical guarantees of WeShap to broader contexts remains an intriguing avenue for future research. Secondly, while WeShap demonstrates robust performance across the evaluated datasets, our analysis reveals that its efficacy is influenced by the underlying dataset's smoothness. Future investigations could explore techniques such as contrastive learning \cite{khosla2020supervised} to enhance data smoothness, potentially improving LF evaluation accuracy.
\section{Related Work}
\textbf{Programmatic Weak Supervision.} In the programmatic weak supervision framework \cite{ratner2016data, ratner2017snorkel, zhang2022survey}, users design labeling functions (LFs) that come from various sources to label large datasets efficiently, such as heuristics \cite{ratner2017snorkel,yu2020fine}, pre-trained models \cite{bach2019snorkel,smith2022language}, external knowledge bases \cite{hoffmann2011knowledge,liang2020bond}, and crowd-sourced labels \cite{lan2019learning}. 
Researchers have also explored approaches to automate the LF design process \cite{varma2018snuba, smith2022language, guan2023can, huangscriptoriumws} or guide users to develop LFs more efficiently \cite{hsieh2022nemo,denham2022witan,boecking2020interactive}. On the other hand, there is rich literature on learning the models to aggregate LFs and de-noise the weak labels \cite{ratner2017snorkel,ratner2019training,fu2020fast,varma2019multi,wu2022learning,zhang2021creating,zhang2022binary,awasthi2019learning}. 
While the work on programmatic weak supervision is abundant, few works focused on the rigorous evaluation of LFs. 
Work on LF design \cite{hsieh2022nemo,guan2023can,hancock2018training, boecking2020interactive} usually use simple heuristics like empirical accuracy or coverage to evaluate and prune out LFs, which does not support fine-grained analysis and revision. The most relevant work is WS Explainer (SIF) \cite{zhang2022understanding}, which leverages the influence function \cite{koh2017understanding} to evaluate the influence of each weak supervision source. 

\textbf{Shapley Values. }  Shapley value \cite{shapley1953value}, originated from game theory and has been applied in machine learning tasks including feature selection \cite{cohen2007feature,sun2012feature,williamson2020efficient}, data evaluation \cite{jia2019towards,ghorbani2019data,shim2021online,kwon2021efficient}, deep learning explanation \cite{chen2018shapley,ancona2019explaining,zhang2021interpreting} and federated learning \cite{liu2022gtg}. The wide application of the Shapley value is credited to its favorable properties that include fairness, symmetry, and efficiency \cite{chalkiadakis2022computational,rozemberczki2022shapley}. However, the high computational complexity of the Shapley value needs to be addressed before applying it in practice. Common approaches to efficiently computing approximate Shapley values include Monte-Carlo permutation sampling \cite{castro2009polynomial,maleki2013bounding,castro2017improving,burgess2021approximating}, multilinear extension \cite{mitchell2022sampling,okhrati2021multilinear} and linear regression approximation \cite{lundberg2017unified,covert2021improving}. While computing the exact Shapley value is infeasible in most cases, it is possible in specific settings. Specifically, Jia et al. \cite{jia2018efficient} demonstrated that the Shapley value for data evaluation can be efficiently computed for nearest-neighbor algorithms. 
\section{Conclusions}
In our study, we propose WeShap values as an innovative method for assessing and refining Programmatic Weak Supervision (PWS) sources. We demonstrate notable computational efficiency and versatility across various datasets and PWS configurations. 
Our results unveil an average downstream model accuracy enhancement of 5.0 points compared to conventional methods, highlighting the pivotal contribution of WeShap values in the progression of machine learning models.

\bibliographystyle{ACM-Reference-Format}
\bibliography{reference_weshap}

\end{document}